\documentclass[letterpaper]{article}

\usepackage{natbib,alifeconf}  

\usepackage[utf8]{inputenc}
\usepackage{natbib}

\usepackage[linktocpage=true]{hyperref}
\hypersetup{
  colorlinks   = true,    
  urlcolor     = blue,    
  linkcolor    = red,    
  citecolor    = red      
}
\usepackage{amsthm}
\usepackage{amsmath,amssymb, amstext}
\usepackage{mathabx}
\usepackage{bm}
\usepackage[font={small,it}]{caption,subcaption}
\usepackage{graphicx}
\usepackage{tikz}
\usetikzlibrary{arrows,fit,matrix,backgrounds,decorations.pathreplacing,calc}
\usetikzlibrary{positioning}
\usepackage{bm} 

\newbox\mybox
\usepackage{enumitem}

\newlist{thmlist}{enumerate}{1}
\setlist[thmlist]{label=(\roman{thmlisti}),noitemsep}

\usepackage{thmtools}

\declaretheorem[
    name=Definition]{mydef}

\declaretheorem[
    name=Theorem]{thm}

\usepackage[capitalize]{cleveref}
\crefname{paragraph}{paragraph}{paragraphs}
\Crefname{paragraph}{Paragraph}{Paragraphs}

\makeatletter

\Crefname{thm}{Theorem}{Theorems}
\Crefname{lem}{Lemma}{Lemmas}
\Crefname{mydef}{Definition}{Definitions}
\Crefname{listthm}{Theorem}{Theorems}
\Crefname{listlem}{Lemma}{Lemmas}
\Crefname{listdef}{Definition}{Definitions}

\addtotheorempostheadhook[thm]{\crefalias{thmlisti}{listthm}}
\addtotheorempostheadhook[lem]{\crefalias{thmlisti}{listlem}}
\addtotheorempostheadhook[mydef]{\crefalias{thmlisti}{listdef}}

\creflabelformat{listthm}{#2\thethm.#1#3}
\creflabelformat{listlem}{#2\thethm.#1#3}
\creflabelformat{listdef}{#2\thethm.#1#3}

\tikzset{%
    myarrow/.style={%
        thick,->,>=stealth',shorten >=2pt
    }
}

\bmdefine\bcolon{:}

\DeclareMathOperator{\pa}{pa}

\newcommand{\bs}{\setminus}
\newcommand{\Nplus}{\mathbb{N}^+}

\newcommand{\X}{\mathcal{X}}

\newcommand{\E}{\mathcal{E}}
\newcommand{\M}{\mathcal{M}}

\renewcommand{\S}{\mathcal{S}}

\newcommand{\Xv}{\{X_i\}_{i \in V}}
\newcommand{\Xw}{\{X_i\}_{i \in W}}

\newcommand{\At}{\{A_t\}_{t \in T}}

\newcommand{\Et}{\{E_t\}_{t \in T}}
\newcommand{\Mt}{\{M_t\}_{t \in T}}
\newcommand{\St}{\{S_t\}_{t \in T}}

\newcommand{\pet}{{\preceq t}}

\newcommand{\xpt}{x_{\prec t}}

\newcommand{\Ent}{\mathfrak{E}}

\newcommand{\perstp}{\mathfrak{S}}

\newcommand{\ext}{w}
\newcommand{\pal}{{PA}}

\DeclareMathOperator{\HS}{H}

\newcommand{\st}{{\succ t}}

\title{Action and perception for spatiotemporal patterns}
\author{Martin Biehl$^{1,2}$ \and Daniel Polani $^{2}$ \\
\mbox{}\\
$^1$Araya, Tokyo, 151-0001\\
$^2$University of Hertfordshire, Hatfield, AL10 9AB \\
martin@araya.org} 

\begin{document}
\setlist{noitemsep,topsep=1pt,parsep=0pt,partopsep=0pt}
\maketitle

\begin{abstract}
This is a contribution to the formalization of the concept of \textit{agents} in multivariate Markov chains. Agents are commonly defined as entities that act, perceive, and are goal-directed. In a multivariate Markov chain (e.g.\ a cellular automaton) the transition matrix completely determines the dynamics. This seems to contradict the possibility of acting entities within such a system. Here we present definitions of actions and perceptions within multivariate Markov chains based on \textit{entity-sets}. Entity-sets represent a largely independent choice of a set of spatiotemporal patterns that are considered as all the entities within the Markov chain. For example, the entity-set can be chosen according to operational closure conditions or complete specific integration. Importantly, the perception-action loop also induces an entity-set and is a multivariate Markov chain. We then show that our definition of actions leads to non-heteronomy and that of perceptions specialize to the usual concept of perception in the perception-action loop. 
\end{abstract}

\section{Introduction}
The perception-action loop (PA-loop) has been used to formalize, in mostly information theoretic terms, various properties associated to agents. These include empowerment \cite{klyubin_empowerment_2005}, autonomy \citep{bertschinger_autonomy_2008}, decisions \citep{tishby_information_2011}, and embodiment \citep{zahedi_quantifying_2013}. 

In the literature agents are usually seen as entities that act, perceive, and are in some way goal-directed \citep[cmp.][]{barandiaran_defining_2009}.

The PA-loop assumes that the entities that make up agents as well as their environments can be captured by interacting stochastic processes. This is a convenient assumption since actions and perceptions can then be easily identified with the interactions (see \cref{sec:paloop}). Further requirements (e.g.\ autonomy) can then be introduced to distinguish stochastic processes that actually constitute agents.  


It has not been formally established whether the assumption that the set of entities that contains agents can be represented by stochastic processes is justified. We have argued in previous work, that (naively) using stochastic processes to capture entities within a given multivariate Markov chain (for example cellular automata like the Game of Life) fails to account for essential proporties of agents irrespective of chosen additional conditions. 
Instead, we have argued for the use of \textit{spatiotemporal patterns} \citep[STPs, previously employed by][]{beer_cognitive_2014,beer_characterizing_2014} to represent entities. 
The immediate advantage is that STPs 
are a superset of structures like gliders in the game of life, spots in reaction diffusion systems 
\citep{virgo_thermodynamics_2011,froese_motility_2014,bartlett_emergence_2015},
and particle based systems exhibiting individuation into multi-particle ``cells'' 
\citep{schmickl_how_2016}. Formally capturing these structures then becomes a matter of selecting the according subsets of STPs i.e.\ the entity-sets (see below).

A disadvantage of the STP based entities is that they lack the formal construction/interpretation of actions and perceptions that the PA-loop provides. As far as we know, no formal definitions of actions and perception exists for STP based entities. The first contribution of this paper are proposals of such formal definitions called \textit{entity action} (\cref{sec:actions}) and \textit{entity perception} (\cref{sec:perceptions}).   

The second contribution is a formal connection between the PA-loop and our STP-based entity actions and perceptions (\cref{sec:palooprelation}). This connection is achieved via the notion of \textit{entity-sets}. This is just the set of those STPs in a multivariate Markov chain that are considered as entities according to an independently specified criterion (e.g.\ organizational closure \citep{beer_characterizing_2014} or complete local integration \citep{biehl_specific_2017,biehl_towards_2016}). Importantly, we can quite naturally identify an entity-set for the PA-loop and use our definitions of entity action and entity perception for this entity-set. The result is that our entity perception coincides with the standard notion of perceptions in the PA-loop and that entity actions are a necessary and sufficient condition for non-heteronomy. Non-heteronomy (i.e.\ not being determined by the environment) was proposed in \citet{bertschinger_autonomy_2008} as part of an information theoretic measure of autonomy.   

The most closely related work is that of \citet{beer_cognitive_2014}. Apart from a generalization to stochastic systems our set of entity perceptions seems to be a straightforward (but surprisingingly tedious) formalization of the cognitive domain of STPs described for the glider in game of life. 
We use ``perceptions'' instead of ``cognitive domain'' only because our motivation came more from the PA-loop and unlike \citet{beer_cognitive_2014} not directly from autopoiesis. 
Concerning the entity actions we deviate from \citet{beer_cognitive_2014} by requiring more from an action than just the continuation of the entity. 



We note that \citet{ikegami_uncertainty_1998} propose to use possible/compatible counterfactual trajectories of game players as signs of autonomy. 
We construct the capability to act from the counterfactual trajectories and find that they imply non-heteronomy which is a related to autonomy. 
\section{Notation}
We restrict ourselves to finite, time-discrete, multivariate Markov chains. These are unrolled in time and can then be formally described as Bayesian networks (BNs). We index the random variables in the BN via the index set $V = J \times T$ where $T$ is the set of all timesteps and $J$ is the set of (spatial) degrees of freedom. If $i \in V$ is an index we also write $i = (j,t) \in J \times T$ where convenient. The BN is then a set $\Xv$ of random variables together with a set of edges determining the parents $\pa(i)=\pa(j,t)$ of each node $i \in V$, and associated mechanisms $p_{(j,t)}(x_{j,t}|x_{\pa(j,t)})$. We assume that the parents of any node are a subset of the nodes at the previous timestep $\pa(j,t)\subseteq (J,t):=\{(j,t)|j \in J\}$. We write $X_A:=(X_i)_{i \in A}$ for the joint random variable consisting of random variables indexed by elements of $A \subseteq V$. We also sometimes write $A_t:=\{(j,t)| j \in A\}$ for the elements in $A$ that correspond to indices at timestep $t$. We refer to $A_t$ as the \textit{time-slice} of $A$ at $t$. The state space of random variable $X_i$ is denoted as $\X_i$ and the specific values are denoted by lower case letters $x_i,y_i,\hat{x}_i,...\in \X_i$. For joint random variables we write $\X_A:=\prod_{i \in A} \X_i$ and $x_A, y_A, \hat{x}_A,...\in \X_A$.

A \textit{spatiotemporal pattern} (STP) is a value $x_A \in \X_A$ of a joint random variable $X_A$ with $A \subseteq V$. Since $A \subseteq V$ is an arbitrary subset of $V$ a STP can specify the values of random variables at multiple timesteps and multiple spatial locations. The set of all STPs is: 
$\bigcup_{A \subseteq V} \X_A=\{x_A \in \X_A|A \subseteq V\}$
. It is important to envision the difference between the set of all STPs and the set of all subsets of random variables of $\Xv$. The latter is isomorphic to the power set of $V$ and can be written as $\bigcup_{A \subseteq V} \{X_A\} = \{X_A|A \subseteq V\}$ and is a set of random variables not a set of \textit{values} of random variables. 

A \textit{trajectory} is a STP $x_V$ that occupies all random variables in the BN $\Xv$. We then say that a STP $y_A \in \X_A$ \textit{occurs within trajectory} $x_V$ if $y_A = x_A$.  

An \textit{entity-set} $\Ent(\Xv)$ is a subset of all STPs. One choice would be to use the entire set of STPs as the entity set. Other choices include using organizational closure conditions like in \citet{beer_cognitive_2014} or the complete specific integration criterion \citep{biehl_towards_2016,biehl_specific_2017}. The following definitions and theorems all assume that $\Xv$ is a multivariate Markov chain with $V=J\times T$ and $\Ent$ is a given entity-set.


\section{Perception-action loop}
\label{sec:paloop}
Given two interacting stochastic processes (e.g.\ \cref{fig:paloop}) we can always extract random stochastic process that explicitly represent the interactions. If we see one of the processes as the agent's memory process $\Mt$ and the other as the environment process $\Et$ then these extracted random variables can be seen as the perceptions and actions of the agent. Perceptions $\St$ then capture exactly all the influence of the environment on the agent and actions $\At$ capture the influence of the agent on the environment. 
\begin{figure}
\begin{center}
  \begin{tikzpicture}

  \matrix (m) [matrix of math nodes, nodes in empty cells,row sep=1.5cm,column sep=1.5cm]
  {
     E_0 &      E_1 &      E_2 & \vphantom{E_3} \\
     M_0 &      M_1 &      M_2 & \vphantom{M_3} \\
   };
  \path[myarrow]
    (m-1-1) edge node {} (m-2-2)
            edge node {} (m-1-2)
    (m-2-1) edge node {} (m-2-2)
            edge node {} (m-1-2)
    

    (m-1-2) edge node {} (m-2-3)
            edge node {} (m-1-3)
    (m-2-2) edge node {} (m-2-3)
            edge node {} (m-1-3)
            
    (m-1-3) edge[-,dotted] node {} (m-2-4)
            edge[-,dotted] node {} (m-1-4)
    (m-2-3) edge[-,dotted] node {} (m-2-4)
            edge[-,dotted] node {} (m-1-4)

%
    ;            
\end{tikzpicture}
  \caption{First timesteps of the PA-loop BN. 
The processes represent environment $\Et$, 
and agent memory $\Mt$. 
}
  \label{fig:paloop}
\end{center}
\end{figure}
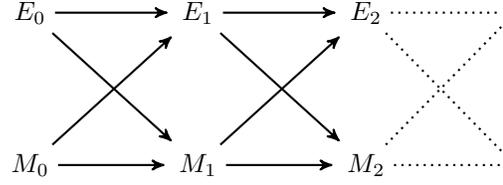
This means we introduce another BN containing two more processes, the action process $\At$ and the sensor process $\St$. The result of this extraction is an extended BN (\cref{fig:expaloop}) with identical joint probability distribution over the two initial stochastic processes $\Mt$ and $\Et$. This perception action loop is used for example in \citet{bertschinger_autonomy_2008}.

The idea behind the extraction of perceptions $\St$ (and conversely actions $\At$) is to partition the state space $\E_t$ of the environment at $t$ into blocks having identical influence on the next memory state $M_{t+1}$. These blocks are then the possible perceptions i.e.\ the states of $S_t$. Formally:

\begin{mydef}
\label{def:paloopperception}
  For each time $t \in T$ and $\hat{e}_t, \bar{e}_t \in \E_t$ let 
  \begin{align}
  \begin{split}
\label{eq:perceptionequiv}
     &\hat{e}_t \equiv_{\epsilon_t} \bar{e}_t \\ 
     &\Leftrightarrow \forall m_{t+1} \in \M_{t+1}, m_t \in \M_t :\\ 
     &\phantom{\Leftrightarrow}\;\; p_{M_{t+1}}(m_{t+1} |  m_t,\hat{e}_t) = p_{M_{t+1}}(m_{t+1} | m_t, \bar{e}_t).
  \end{split}   
  \end{align} 
  Then:
  \begin{thmlist}
  \item The \textit{sensor partition\footnote{The construction of the sensor partition is not new. It is also used for example in \citet{balduzzi_detecting_2011} to obtain coarser states (alphabet) of joint random variables. 
  The authors thank Benjamin Heuer for originally pointing them to this construction.} 
  $\epsilon_t$} is then defined as the set of equivalence classes of the equivalence relation $\equiv_{\epsilon_t}$.
  \item The \textit{set of sensor values} is defined as $\S_t := \epsilon_t$ and an element $s_t \in \S_t$ (which is also a block in $\epsilon_t$ is called a \textit{perception} of a \textit{sensor value}.
  \end{thmlist}
\end{mydef}

%
%


In the symmetrical way we define actions via a partition of $\M_t$ and arrive at the extended PA-loop of \cref{fig:expaloop}. It is then straightforward to prove the following theorem: 

%
\begin{figure}
\begin{center}
  \begin{tikzpicture}

  \matrix (m) [matrix of math nodes, nodes in empty cells,row sep=.5cm,column sep=.5cm]
  {
     E_0 &     & E_1 &     & E_2 & \vphantom{E_3} \\
         & S_0 &     & S_1 &     & \vphantom{S_2} \\
         & A_0 &     & A_1 &     & \vphantom{A_2} \\
     M_0 &     & M_1 &     & M_2 & \vphantom{M_3} \\
   };
  \path[myarrow]
    (m-1-1) edge node {} (m-2-2)
            edge node {} (m-1-3)
    (m-2-2) edge node {} (m-4-3)
    (m-3-2) edge node {} (m-1-3)
    (m-4-1) edge node {} (m-3-2)
            edge node {} (m-4-3)

    (m-1-3) edge node {} (m-2-4)
            edge node {} (m-1-5)
    (m-2-4) edge node {} (m-4-5)
    (m-3-4) edge node {} (m-1-5)
    (m-4-3) edge node {} (m-3-4)
            edge node {} (m-4-5)
            
    (m-1-5) edge[-,dotted] node {} (m-2-6)
            edge[-,dotted] node {} (m-1-6)
    (m-4-5) edge[-,dotted] node {} (m-3-6)
            edge[-,dotted] node {} (m-4-6)
    
    ;            
\end{tikzpicture}
  \caption{First time-steps of the BN of the extended PA-loop. The processes $\At$ and $\St$ mediate all interactions between $\Mt$ and $\Et$ without changing the probability distributions over the latter (see \cref{thm:paloopext}).}
  \label{fig:expaloop}
\end{center}
\end{figure}
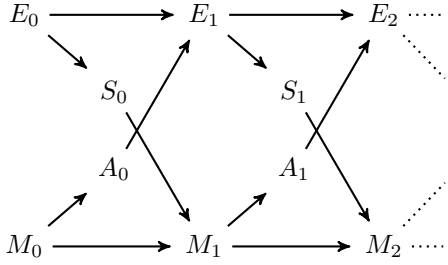

\begin{thm}[Invariant extension theorem]
\label{thm:paloopext}
Given a perception action loop $\Xv =\{M_t,E_t\}_{t \in T}$ and its extended PA-loop $\Xw=\{M_t,A_t,S_t,E_t\}_{t \in \Nplus}$. Let $p_V=p_{M_T,E_T}$ be the probability distribution over the entire perception action loop $\Xv$ and let $p^{\ext}_{M_T,E_T}$ be the marginal probability distribution over the memory and environment process obtained from the probability distribution $p^{\ext}_W$ over the entire extended PA-loop. Then
\begin{equation}
p_{M_T,E_T} = p^{\ext}_{M_T,E_T}.
\end{equation} 
\end{thm}
\begin{proof}
  This is probably fairly well known but see \citet{biehl_formal_2017} for an explicit proof.
\end{proof}

%
%
  This shows that the introduction of action and sensor process in the above way only makes the interactions between agent and environment processes explicit and does not introduce any additional dynamics.
  The theorem also shows that the sensor process (and conversely the actions) captures all influences from the environment on the agent. Else the dynamics of the original processes could not remain identical. 
  In \cref{sec:perceptions} we want to capture all influences of the environment on a set of STPs / entities instead of on a stochastic process like $\Mt$. This will require a generalization of the perception extraction procedure in \cref{def:paloopperception}.

\section{Entity action}
\label{sec:actions}

We now define a concept of actions for a given entity-set 
in a multivariate Markov chain. First we briefly sketch the main ideas behind the definition.
\label{sec:contrast}
Due to our setting of a given multivariate Markov chain that actions have to occur within our concept of actions differs from other approaches.
Paraphrasing \citet{sep-action} only slightly, what distinguishes actions among events is that they do not merely happen to individuals but rather that they are \textit{made to happen by} the individuals. 

%
This is problematic in our setting where STPs (as entities) take the role of individuals. What ``happens'' in a multivariate Markov chain are the trajectories and the STPs occurring within them. The Markov chain's dynamics are determined by its mechanisms $p_{j,t}$ with $j \in J, t \in T$. These in turn determine (possibly stochastically) what is going to happen at all times anywhere within the chain. 

 Therefore, it is impossible that a multivariate Markov chain contains an STP or entity that can make something happen beyond what happens anyway due to the mechanisms. This means we have to explain and define actions in a different way. 

We should also note that unlike other accounts of actions \citep{sep-action} we do not require actions to be necessarily purposeful or goal-directed in any way. 
The way we conceive agency an entity with actions will be considered goal-directed if its actions are goal-directed in some sense. 




\label{sec:background}

To get to our own definition of actions we 
note that events called actions are usually attributed to a limited or bounded region or part of the universe e.g.\ the body of a living organism or sometimes just its brain if it has one. These parts usually contain mechanisms or configurations of matter that are either 
$a)$ not directly observable to a human observer e.g.\ hidden in an opaque container, or $b)$
not well understood by the human observer, or $c)$ both.
These factors inevitably lead to unpredictability of such events. In other words, events that are attributed to well understood and therefore predictable mechanisms, e.g.\ sunrises, are not considered actions. 

From this point of view actions are not, beyond their possibly complex and unobserved origin, special events but may \textit{appear} as special to observers that lack the sensory and computational capacity to resolve or understand them. 
In our approach we construct actions as events in such a way that they are fundamentally unpredictable by any observer within the system. As we will see this can be done without the need for a definition of an observer.

Note that this approach remains compatible with observer-dependent notions of apparent or as-if actions \citep[cmp.][]{mcgregor_bayesian_2017}. The ``fundamental'' actions are apparent actions for every possible observer while other events are actions for some observers and plain, predictable events for others.

Up to now we have ignored randomness. True randomness (in the sense of stochastic independence of the event from any other event in the universe), if it exists in a system, can never be explained, predicted, or understood. Combined with our reasoning above this suggests that all random events are actions and fundamentally so. 

This is against our intuition, random events should not be seen as actions of agents. However, we place the burden of ruling out random events from being interpreted as actions \textit{of agents} on the entity-set and/or goal-directedness. We expect useful notions of entities not to consist of (completely) independent events which would prevent such events from becoming actions in our approach. Furthermore independent events seem not to be useful in order to achieve a particular goal. However, the usefulness of random number generators might be seen as a counterexample. As of now we have no formal definition of goal-directedness such that answering these questions is future work.

\label{sec:actiondef} 
When we want to define actions for entities the first issue we run into is that entities are already fixed STPs. 

Therefore each entity $x_A \in \Ent$ already consists of the consequences of whatever actions (if any) it took. It is in this sense the \textit{result} of its own actions. There is no freedom left. In order to investigate the actions we therefore have to deconstruct the entity and see what other actions it \textit{could} have taken time-slice by time-slice. For this we look for counterfactual entities $y_B$ whose time-slice at $t$ can occur in exactly the same \textit{environment} $x_{V_t \bs A_t}$ and has a different immediate future. Requiring the co-occurrence with the same environment makes the different futures unpredictable for anything part of that environment. No observer can therefore predict these futures either. Even an observer that distinguished the two entities in the past is in the same state at $t$ independent of the entity it is faced with. Therefore it must have forgotten their difference. In this way we get a definition of actions as unpredictable events for observers without needing to formally define observers.

%

For the formal definition we first define:
\begin{mydef}[Environment of an STP]
  Let $x_A$ be a STP. Then the environment of $x_A$ at time $t$ is the spatial pattern $x_{V_t \bs A_t}$.
\end{mydef}

We now state the definition of an action of an entity at a time $t$ in a particular trajectory formally.

\begin{mydef}[Action and co-action of an entity]
\label{def:action}
 Let $x_V \in \X_V$ with $p_V(x_V) > 0$. Also let $x_A$ be an entity with non-empty time-slices at $t,t+1$. Then $x_A$ \textit{performs an action $x_{A_{t+1}}$ at time $t$ in trajectory $x_V$} if there exists an entity $y_B$ with non-empty time-slices at $t,t+1$ such that 
 \begin{thmlist}
   \item $y_B$ occurs in $y_V \neq x_V$ with $p_V(y_V)>0$, 
   \item at $t$ the entities $x_A$ and $y_B$ occupy the same random variables: $B_t = A_t$,
   \item at $t$ the trajectories $x_V$ and $y_V$ are otherwise identical: $x_{V_t \bs A_t} = y_{V_t \bs A_t}$,
   \item at $t+1$ the entities are different: $x_{A_{t+1}} \neq y_{B_{t+1}}$.
 \end{thmlist}
 We also call $y_B$ a \textit{co-action entity}, $y_V$ a \textit{co-action trajectory}, and $y_{B_{t+1}}$ a \textit{co-action}.
\end{mydef}
Note that all requirements are symmetric. Therefore, if $x_A$ performs an action then its co-action also performs an action. 
Also, the notion of co-action entities can easily be extended to more than one co-action entity. We only have to make sure that all entities in a set of co-action entities are mutually different at $t+1$.
Furthermore, it is easy to generalise the definition of actions to situations where $x_A$ and $y_B$ must occupy the same variables for an interval of time $[t-m:t]$ before the action. In that case, the environment $x_{V_{[t-m:t]}} \bs A_{[t-m:t]}$ must be identical during this interval.

Finally, note that the condition that the two acting entities differ at time $t+1$ can be fulfilled in two ways. If $A_{t+1} \neq B_{t+1}$ then we call these actions \textit{extent actions}. Else, if the actions differ only in value i.e.\ we have 
$A_{t+1}=B_{t+1}$ so that $x_{A_{t+1}} \neq y_{A_{t+1}}$ then we call these actions \textit{value actions}.

The difference between value actions and extent actions is made possible due to our definition of entities as STPs. 
An intriguing question for the future is whether the capabilities of agents to act both in value and extent are truly superior to agents that only act in value such as those modelled by PA-loops. 
As we will see in \cref{sec:palooprelation} probabilistic and information theoretic expressions are easy to formulate for actions in value only. However, for actions in extent this has not been done yet.

\section{Entity perception}
\label{sec:perceptions}
In this section we formally define perception for (STP-) entities. We make no distinction here between perception, experience, and sensory input. In the tradition of modelling agent-environment systems using dynamical systems or their probabilistic generalisations as stochastic processes in PA-loops we define perception as \textit{all effects} that the environment has on an individual/agent \citep{beer_dynamical_1995}. 


We run into a similar problem as with the actions. An entity is already a fixed STP that contains all influence that it may have been subjected to. It is in this sense the \textit{result} of influence (or no influences) from its surroundings.  In order to investigate these influences we therefore have to deconstruct the entity and see how it was ``formed'' by external influences / perceptions time-slice by time-slice. 

The idea is to fix the past of the entity up to $t$ and use the set of counterfactual entities with the same past as the alternative futures. We can then partition the possible environments of these counterfactual entities according to their influence on the probability distribution over the entities' futures. This is done in basically the same way as we defined perception in the PA-loop in \cref{def:paloopperception}. However, there are some technical issues to overcome.

 The set of entities with identical pasts up to time $t$ can be interpreted as the set of entities that are the most like $x_A$ up to $t$. These are \textit{different} entities but they only differ in the future. Their futures (including their next time-slices) are therefore a close analogue to the next states $m_{t+1}$ of agent memories in the PA-loop. To make sure however that the entities have a next time-slice we also require that they have \textit{non-empty} next time-slice. These requirements together define the notion of the \textit{co-perception entities of an entity $x_A$ at time $t$}. These are entities that also perceive something (maybe the same thing) at $t$ (in their trajectories) if $x_A$ perceives something at $t$.    
\begin{mydef}[Co-perception entities of an entity at $t$]
\label{def:copent}
   Let $x_A \in \Ent$ be an entity with non-empty time-slices at $t$ and $t+1$. The set of \textit{co-perception entities $\perstp(x_A,t)$ of entity $x_A$ at $t$} is the set of entities with non-empty time-slices at $t$ and $t+1$, and that are identical up to $t$:
  \begin{equation}
    \perstp(x_A,t):=\{y_B \in \Ent: B_t,B_{t+1} \neq \emptyset, y_{B_\pet}=x_{A_\pet}\}.
  \end{equation} 
\end{mydef}
Next we want to define a conditional probability distribution over these co-perception entities similar to $p_{M_{t+1}}(m_{t+1}|e_t,m_t)$. For this we need a random variable that ranges over all the possible futures of the co-perception entities. The naivest way to do this would be to use for $x^k_{A^k} \in \perstp(x_A,t)$:
\begin{equation}
\label{eq:naivemorph}
  p_{\perstp}(x^k_{A^k_\st} | x_{V_t \bs A_t}, x_{A_\pet}):=\frac{p_{A^k,V_t \bs A_t}(x^k_{A^k_\st}, x_{V_t \bs A_t},x_{A_\pet})}{p_{V_t \bs A_t, A_\pet}(x_{V_t \bs A_t},x_{A_\pet})}.
\end{equation} 
However, this has two problems. The first is that in general the denominator may vanish for some environments $x_{V_t \bs A_t} \in \X_{V_t \bs A_t}$. The second is that it is not a conditional probability since the sum over all co-perception indices is not necessarily one. These problems can be solved by restricting the set of environments and by introducing a mutual-exclusivity condition on the futures of the co-perception entities. 

To restrict the environments we therefore define the \textit{co-perception environments} in the following way.
\begin{mydef}[Co-perception environments]
\label{def:copenv}
   Let $x_A \in \Ent$ be an entity with non-empty time-slices at $t$ and $t+1$ and $\perstp(x_A,t)$ its co-perception entities. Then define the \textit{associated co-perception environments $\X^\perstp_{V_t \bs A_t} \subseteq \X_{V_t \bs A_t}$} by 
  \begin{align}
  \begin{split}
   \label{eq:copenv}
    &\X^\perstp_{V_t \bs A_t}:= \\
    &\{\bar{x}_{V_t \bs A_t} : \exists y_B \in \perstp(x_A,t), p_{B,V_t\bs A_t}(y_B,\bar{x}_{V_t \bs A_t}) > 0\}.
  \end{split}
  \end{align} 
\end{mydef}
  The co-perception environments of a co-perception set $\perstp(x_A,t)$ are then the spatial patterns $\X_{V_t \bs A_t}$ at $t$ that can co-occur with \textit{at least one} co-perception environment. 
By definition of the co-perception environments the denominator of \cref{eq:naivemorph} cannot vanish anymore if we require that the construction only allows co-perception environments. However, in order for the sum over all co-perception entities to equal one (for all co-perception environments) we need to have 
\begin{align}
\begin{split}
  \sum_k p_{A^k,X_{V_t \bs A_t}}&(x^k_{A^k_\st}, x_{V_t \bs A_t},x_{A_\pet}) \\
  &= p_{V_t \bs A_t, A_\pet}(x^k_{A^k_\st}, x_{V_t \bs A_t},x_{A_\pet}).
\end{split}
  \end{align} 
This is the case in general if for all $x_{V_t \bs A_t} \in \X^\perstp_{V_t bs A_t}$ and all co-perception entities' futures are mutually exclusive, i.e.\ for all $x^k_{A^k}\neq x^l_{A^l} \in \perstp(x_A,t)$ we have 
\begin{equation}
\label{eq:specmuex}
  \Pr(x^k_{A^k},x^l_{A^l},x_{V_t \bs A_t},x_{A_\pet})=0.
\end{equation} 

\label{sec:nonint}
%
%
%
%
%
%
This condition can be guaranteed if we require a form of non-interpenetration. 
This condition on entity-sets states that there cannot be two different entities which are identical up to some point in time $t$ and then, in the same \textit{single} trajectory (with positive probability), at some point ``reveal'' their difference. If entities with identical pasts \textit{ever} reveal their difference they must be in different trajectories i.e.\ they must be mutually exclusive. 
\begin{mydef}[Non-interpenetration]
\label{def:noninterpen}
   An entity-set $\Ent \subseteq \bigcup_{B \subseteq V} \X_B$ \textit{satisfies non-interpenetration or is non-interpenetrating} if for all $y_B,z_C \in \Ent$ we have
  \begin{align}
    \begin{split}\exists t &\in T : y_{B_\pet} = z_{C_\pet} \text{ and }   y_{B_{t\prec}} \neq z_{C_{t\prec}} \\
&\Rightarrow   \Pr(X_{B_{t\prec}}=y_{B_{t\prec}},X_{B_{t\prec}}=z_{C_{t\prec}}|y_{B_\pet})=0.
    \end{split}
  \end{align} 
\end{mydef}

Non-interpenetration implies that co-perception entities are mutually exclusive:
\begin{thm}
\label{thm:nonintmuex}
   Let $x_A \in \Ent$ be an entity with non-empty time-slices at $t$ and $t+1$ and $\perstp(x_A,t)$ its co-perception entities. If $\Ent$ satisfies non-interpenetration then $\perstp(x_A,t)$ is mutually exclusive.
\end{thm}
\begin{proof}
  Let $y_B,z_C \in \perstp(x_A,t)$ with $y_B \neq z_C$. Then they have identical pasts and so we have $y_{B_\pet} = z_{C_\pet}$. From non-interpenetration we then get 
  \begin{equation}
    \Pr(X_B=y_B,X_C=z_C)=0.
  \end{equation} 
  This is stronger than \cref{eq:specmuex}.
%
\end{proof}

This means that under non-interpenetration \cref{eq:naivemorph} is a well defined conditional probability distribution 
. 
However, this conditional probability distribution is still quite different from $p_{M_{t+1}}(m_{t+1}|e_t,m_t)$ since it ranges over the entire futures $x^k_{A^k_\st}$ of the co-perception entities and not just next timesteps. 

At each transition from time-step $t$ to $t+1$ the co-perception entities $\perstp(x_A,t)$ split up into sets of entities that are identical up to $t+1$ (we will call these sets the \textit{branches}). Only one of these sets is the set $\perstp(x_A,t+1)$. For example an entity $y_B \in \perstp(x_A,t)$ with the same past up to $t$ but with a different time-slice at $t+1$ i.e.\ $y_{B_{t+1}} \neq x_{A_{t+1}}$ is part of a different branch. In that case this branch is $\perstp(y_B,t+1)$ and we have $\perstp(y_B,t+1) \cap \perstp(x_A,t+1) = \emptyset$. In summary then the dynamics of the system split up the co-perception entities of $x_A$ up to $t$ into disjoint sets (the branches) of entities with identical pasts up to $t+1$. We can then interpret the branches at the time $t+1$ as the distinctions among the co-perception entities that are revealed at time $t+1$. Further distinctions among the co-perception entities are only revealed at later times. This also means that these are \textit{all} differences that could possibly be due to the influence of the environment at $t$ and that show their effect at $t+1$ (not later). In this way the perceptions at $t$ should also be defined with respect to these branches. We call the partition that is defined via the identification of entities in $\perstp(x_A,t)$ that are identical up to $t+1$ the \textit{branching partition}. 
\begin{mydef}[Branching partition]
\label{def:compart}
 Let $x_A \in \Ent$ be an entity with non-empty time-slices at $t$ and $t+1$ and $\perstp(x_A,t)$ its co-perception entities. 
Then define the \textit{branching partition $\eta(x_A,t)$ of $\perstp(x_A,t)$} as the partition induced by the equivalence classes of the equivalence relation
   \begin{equation}
    \begin{split}y_B \sim & z_C \\
    &\Leftrightarrow  y_{B_{t+1}}=z_{C_{t+1}},
    \end{split}
  \end{equation} 
where $y_B,z_C \in \perstp(x_A,t)$.
\end{mydef}
We note that the definition of the branching partition can easily be generalised to more than one time-step into the future. Instead of requiring equality at $t+1$ we can require equality for the next $r$ time-steps.

\label{sec:branchmorph}
Given the branching partition $\eta(x_A,t)$ for a non-interpenetrating entity set we can then define a conditional probability distribution over the branches by just summing up the probabilities of all entities in each branch (remember that they are all mutually exclusive) to get the probability of a branch. This gives us the branch-morph defined below.\footnote{We note here that for entity-sets that exhibit interpenetration we can still define branch-morphs for mutually-exclusive subsets of the co-perception entities. Since the choice of these subsets is arbitrary however this does not lead to a uniquely defined notion of perception. For more details see \citet{biehl_formal_2017}.}

\begin{mydef}[Branch-morph]
\label{def:branchmorph}
 Let $\Xv$ be a multivariate Markov chain with index set $V=J\times T$ and entity set $\Ent$. Let $x_A \in \Ent$ be an entity with non-empty time-slices at $t$ and $t+1$ and $\perstp(x_A,t)$ its co-perception entities and $\eta(x_A,t)$ the branching partition.
 Furthermore, let $\X^\perstp_{V_t \bs A_t} \subseteq \X_{V_t \bs A_t}$ be the associated co-perception environments. Also write for every block $b \in \eta(x_A,t)$:
 \begin{equation}
 \label{eq:beforebranchmorph}
   p(b|\hat{x}_{V_t \bs A_t},x_{A_{\preceq t}}):= \sum_{y_B \in b} p_{B_{t\prec},V_t\bs A_t}(y_{B_{t\prec}}|\hat{x}_{V_t \bs A_t},x_{A_{\preceq t}}).
 \end{equation} 
  Then for each $\hat{x}_{V_t \bs A_t} \in \X^\perstp_{V_t \bs A_t}$ we define the \textit{branch-morph} over $\eta(x_A,t)$ as the probability distribution $p_{\eta(x_A,t)}(.|\hat{x}_{V_t\bs A_t},x_{A_{\preceq t}}):\eta(x_A,t) \rightarrow [0,1]$ with 
 \begin{equation}
 \label{eq:branchmorph}
   p_{\eta(x_A,t)}(b|\hat{x}_{V_t\bs A_t},x_{A_{\preceq t}}):=\frac{p(b|\hat{x}_{V_t \bs A_t},x_{A_{\preceq t}})}{\sum_{c \in \eta(x_A,t)} p(c|\hat{x}_{V_t \bs A_t},x_{A_{\preceq t}})},
 \end{equation} 
 for all $b \in \eta(x_A,t)$.
\end{mydef}

With the branch-morph we can then define, as expected, the perceptions as equivalence classes of the co-perception environments with respect to the associated branch-morph. First we define a partition of the co-perception \textit{environments} called the co-perception environment partition. The perceptions are then the blocks of this partition.

\begin{mydef}
Let $\Xv$ be a multivariate Markov chain with index set $V=J\times T$ and entity set $\Ent$. Let $x_A \in \Ent$ be an entity with non-empty time-slices at $t$ and $t+1$ and $\perstp(x_A,t)$ its co-perception entities and $\eta(x_A,t)$ the branching partition.
 Furthermore, let $\X^\perstp_{V_t \bs A_t} \subseteq \X_{V_t \bs A_t}$ be the associated co-perception environments.
  Then define the \textit{co-perception environment partition $\pi^\perstp(x_A,t)$ of $\X^\perstp_{V_t \bs A_t}$} as the partition induced by the equivalence classes of the equivalence relation
  \begin{equation}
    \begin{split}
    \label{eq:bramorpheq}
    \hat{x}&_{V_t \bs A_t} \sim  \bar{x}_{V_t \bs A_t} \\
    &\Leftrightarrow \forall b \in \eta(x_A,t): \\
    &\phantom{\Leftrightarrow \forall} p_{\eta(x_A,t)}(b|\hat{x}_{V_t\bs A_t},x_{A_{\preceq t}}) = p_{\eta(x_A,t)}(b|\bar{x}_{V_t\bs A_t},x_{A_{\preceq t}}).
    \end{split}
  \end{equation} 
\end{mydef}
  This means all associated co-perception environments in the same block of $\pi^\perstp(x_A,t)$ have the same branch-morph. In other words they lead to the same branch of entity futures (i.e.\ the same future branch) with the same probabilities. Then all elements of these environment blocks have identical effects on the future branches and these branches cannot distinguish between environments within the blocks.

\begin{mydef}[Perceptions]
Let $\Xv$ be a multivariate Markov chain with index set $V=J\times T$ and entity set $\Ent$. Let $x_A \in \Ent$ be an entity with non-empty time-slices at $t$ and $t+1$ and $\perstp(x_A,t)$ its co-perception entities.
 Furthermore, let $\X^\perstp_{V_t \bs A_t} \subseteq \X_{V_t \bs A_t}$ be the associated co-perception environments and $\pi^\perstp(x_A,t)$ its co-perception environment partition.

 Then the blocks of $\pi^\perstp(x_A,t)$ are called the \textit{perceptions of $x_A$ at $t$}.

\end{mydef}

\section{Entity action and perception in the PA-loop}
\label{sec:palooprelation}
%
%
%
%

We now show that agent-environment systems as modelled by the PA-loop are multivariate Markov chains containing a specific choice of entity sets.

In the PA-loop each trajectory $x_V$ is considered to consist of a time-evolution $m_T$ of the agent and a time-evolution of the environment $e_T$. The agent therefore occurs in every trajectory and occupies the same degree of freedom in every trajectory. 
Each of the time-evolutions $m_T$ is a STP in the PA-loop. Since for us entities are STPs we define the entity-set $\Ent^\pal$ of a PA-loop as the set of time-evolutions of the agent process i.e.\
\begin{equation}
  \Ent^\pal:=\{m_T \in \prod_{t \in T} \M_t\}.
\end{equation}
Similarly, we can define entities for the environments and add them to $\Ent^\pal$. 


\paragraph*{Entity actions in the PA-loop}

We can write every trajectory as a pair $(m_T,e_T)$ where $m_T$ is an entity. The entity $m_T$ then performs an entity action at time $t$ in trajectory $(m_T,e_T)$ with $p_{M_T,E_T}(m_T,e_T)>0$ if there is an entity $\bar{m}_t$ such that
\begin{itemize}[noitemsep]
   \item $\bar{m}_T$ occurs in $(\bar{m}_T,\bar{e}_T) \neq (m_T,e_T)$ with $p_{M_T,E_T}(\bar{m}_T,\bar{e}_T)>0$, 
   \item at $t$ entities $m_T$ and $\bar{m}_T$ occupy the same random variables, which is the case for all entities in $\Ent^\pal$,
   \item at $t$ environments of $m_T$ and $\bar{m}_T$ are identical: $e_t = \bar{e}_t$,
   \item at $t+1$ the entities are different: $m_{t+1} \neq \bar{m}_{t+1}$.
\end{itemize}
Since all entities occupy the same random variables we can only have value actions in the PA-loop.


If we assume that these conditions are fulfilled at some time $t$ for two entities $m_T,\bar{m}_T$ we can derive that the conditional entropy $\HS(M_{t+1}|E_t)$ of the next agent state given the current environment state is greater than zero.
To see this note that from $p_{M_T,E_T}(m_T,e_T)>0$ and $p_{M_T,E_T}(\bar{m}_T,\bar{e}_T)>0$ it directly follows that $p_{M_{t+1}}(m_{t+1}|e_t)>0$, $p_{M_{t+1}}(\bar{m}_{t+1}|e_t)>0$ and $p_{E_t}(e_t)>0$. Plugging this into the definition we get $\HS(M_{t+1}|E_t)>0$.
    
It can also be seen that the more different co-action entities there are for a time $t$ the higher the conditional entropy $\HS(M_{t+1}|E_t)$ can get.
The final value of $\HS(M_{t+1}|E_t)$ depends on the actual probabilities but the maximum value for $n$ co-actions is $\log n$. Also note that if there are no actions at $t$ i.e.\ no co-action entity in no co-action trajectory at $t$ then $\HS(M_{t+1}|E_t)=0$. Entity actions of entities in $\Ent^\pal$ are therefore necessary and sufficient for $\HS(M_{t+1}|E_t)>0$.
The conditional entropy $\HS(M_{t+1}|E_t)$ measures the uncertainty about the next agent state when the current environment state is known. It has been proposed as part of an autonomy measure as a measure of non-heteronomy in \citet{bertschinger_autonomy_2008}. Non-heteronomy means that the agent is not determined by the history of the environment. In this terminology entity actions are a necessary and sufficient condition for non-heteronomy.

\paragraph*{Entity perception in the PA-loop}

We now look at how entity perception as defined in \cref{sec:perceptions} specialises to the case of the PA-loop. This argument in effect constitutes a proof that our \cref{def:branchmorph} of the branch-morph is a generalisation of the conditional probability distributions $p_{M_{t+1}}(.|m_t,e_t):\M_{t+1} \rightarrow [0,1]$ to non-interpenetrating arbitrary sets of co-perception entities 
This result is not surprising since we set out to do just this but it is also instructive to work through the recovery of the original expression of the conditional probability distribution starting from the general branch-morph.


We pick an entity $m_T$ from the entity set $\Ent^\pal$ and consider its perceptions at an arbitrary time-step $t\in T$. In order to get the perceptions at $t$ we need 
\begin{enumerate}
  \item the co-perception entities $\perstp(m_T,t)$ of $m_T$ at $t$,
  \item the branching partition $\eta(m_T,t)$ with its branches,
  \item the co-perception environments,
  \item the branch-morphs for each environment,
  \item and the co-perception environment partition $\pi^\perstp(x_A,t)$ with its blocks, the perceptions.
\end{enumerate}
These can be identified in the following way.

$1.$ The co-perception entities $\perstp(m_T,t)$ are the entities in $\Ent^\pal$ that have non-empty time-slices at $t,t+1$, and that are identical to $m_T$ up to $t$. All entities in $\Ent^\pal$ have non-empty time slices at all times. So we have:
\begin{equation}
  \perstp(m_T,t)=\{\bar{m}_T \in \Ent^\pal:\bar{m}_{\preceq t}=m_{\preceq t}\} \\
\end{equation} 

$2.$
First note that the entity set $\Ent^\pal$ satisfies non-interpenetration since they all occupy the same set $\Mt$ of random variables. 
The branching partition $\eta(m_T,t)$ is composed out of blocks (the branches) of co-perception entities that 
%
are identical up to $t+1$ i.e.\  
   \begin{equation}
    \begin{split}\hat{m}_T \sim & \bar{m}_T \\
    &\Leftrightarrow \hat{m}_{t+1}=\bar{m}_{t+1}.
    \end{split}
  \end{equation} 
We can therefore identify the blocks of $\eta(m_T,t)$ i.e.\ the future branches by the values that the entities take at $t+1$. 
Define the branch $b(\bar{m}_{t+1})$ associated to $\bar{m}_{t+1} \in \M_{t+1}$ via 
\begin{equation}
\label{eq:pabranches}
  b(\bar{m}_{t+1}):=\{\hat{m}_T \in \perstp(m_T,t) : \hat{m}_{t+1} = \bar{m}_{t+1}\}.
\end{equation}
The branching partition is then: 
\begin{equation}
  \eta(m_T,t)= \{b(\bar{m}_{t+1})\subseteq \perstp(m_T,t): \bar{m}_{t+1} \in \M_{t+1}\}. 
\end{equation} 

$3.$
The co-perception environments are the STPs $x_{V_t \bs A_t}$ compatible with at least one co-perception entity. For the PA-loop and entity $m_T$ at $t$ we have $\X_{V_t \bs A_t}=\E_t$ and therefore $\X^\perstp_{V_t \bs A_t}= \E^\perstp_t$. Where $\E^\perstp_t$ is 
\begin{equation}
  \E^\perstp_t= \{e_t \in \E_t:\exists \bar{m}_T \in \perstp(m_T,t), p_{M_T,E_t}(\bar{m}_T,e_t)>0\}.
\end{equation} 
If we marginalize over $M_\st$ we can see that this is equivalent to 
\begin{equation}
  \E^\perstp_t= \{e_t \in \E_t: p_{M_t,E_t}(\bar{m}_t,e_t)>0\}.
\end{equation} 

$4.$
The branch-morphs are the probability distributions $p_{\eta(m_T,t)}(.|e_t,m_\pet):\eta(m_T,t)\rightarrow [0,1]$ over the branches for each co-perception environment $e_t \in \E^\perstp_t$. These are defined using \cref{eq:beforebranchmorph} which for the perception-loop becomes
\begin{equation}
   p(b(\bar{m}_{t+1}),e_t|m_\pet):= \sum_{\hat{m}_T \in b(\bar{m}_{t+1})} p_{M_{t\prec},E_t}(\hat{m}_{t\prec},e_t|m_\pet).
 \end{equation} 
We can rewrite the sum on the right hand side using \cref{eq:pabranches} for $b(\bar{m}_{t+1})$ and then $\perstp(m_T,t)_{t \prec}=\M_{t \prec}$:
\begin{align}
 p(b(\bar{m}_{t+1}),e_t|m_\pet) 
  =p_{M_{t+1},E_t}(\bar{m}_{t+1},e_t|m_\pet).
\end{align}
The definition of the branch-morph for the PA-loop is
\begin{align}
   p_{\eta(m_T,t)}(b(\bar{m}_{t+1})|e_t,m_\pet):&=\frac{p(b(\bar{m}_{t+1}),e_t|m_\pet)}{\sum_{b \in \eta(m_T,t)} p(b,e_t|m_\pet)} 
 \end{align}
 which can be rewritten (with some work) as 
\begin{align}
  \label{eq:branchmorphpa}
   p_{\eta(m_T,t)}(b(\bar{m}_{t+1})|e_t,m_\pet)
   &=p_{M_{t+1}}(\bar{m}_{t+1}|e_t,m_\pet) \\
   &=p_{M_{t+1}}(\bar{m}_{t+1}|e_t,m_t)
 \end{align}
where we used the BN of the PA-loop. 

$5.$The co-perception environment partition $\pi^\perstp(m_T,t)$ of $\E^\perstp_t=$ is induced by \cref{eq:bramorpheq} which, using the PA-loop and \cref{eq:branchmorphpa} becomes
 \begin{equation}
    \begin{split}\hat{e}_t & \sim  \bar{e}_t \\
    &\Leftrightarrow \forall m_{t+1} \in \M_{t+1}: \\
    &\phantom{\Leftrightarrow \forall}p_{M_{t+1}}(m_{t+1}|\hat{e}_t,m_t) = p_{M_{t+1}}(m_{t+1}|\bar{e}_t,m_t)
    \end{split}
  \end{equation} 
which is just the equivalence relation of \cref{eq:perceptionequiv} used to extract the sensor-values in \cref{sec:paloop}.

%
So we have seen that our definitions of \cref{sec:perceptions} specialise in the case of the PA-loop to the same concept of perception as in \cref{sec:paloop}. 
\section{Conclusion}
We have defined actions and perceptions for entity-sets and therefore for sets of spatiotemporal patterns. This provides a formally defined way to associate gliders and similar spatiotemporal patterns in reaction-diffusion systems with actions and perceptions. This is a step towards a foramization of agency of such patterns. Here a notion of goal-directedness is still missing and future work. We have also shown how our definitions specialize to a necessary and sufficient condition for non-heteronomy and the standard notion of perceptions of the agent process in the PA-loop. For future research it is interesting to note that the branch-morphs are generalisations of the conditional probability distribution $p_{M_{t+1}}(.|e_t,m_t)$. These conditional probability distributions play a role in various information theoretic concepts formulated for the PA-loop. This suggests we might be able to translate bakc and forth between PA-loop concepts and those for spatiotemporal patterns in the future.

We noted that a unique definition of entity perception is dependent on the condition of non-interpenetration of the entity-set.

\footnotesize
\bibliographystyle{apalike}
\bibliography{../bibliography} 

\end{document}